\documentclass[nohyperref]{article}

\usepackage{microtype}
\usepackage{graphicx}
\usepackage{subfigure}
\usepackage{booktabs} 

\usepackage{hyperref}

\usepackage[accepted]{icml2023}

\usepackage{amsmath}
\usepackage{amssymb}
\usepackage{mathtools}
\usepackage{amsthm}

\usepackage[capitalize,noabbrev]{cleveref}

\usepackage{listings}
\usepackage{xcolor}
\definecolor{textblue}{rgb}{.2,.2,.7}
\definecolor{textred}{rgb}{0.54,0,0}
\definecolor{textgreen}{rgb}{0,0.43,0}
\usepackage{listings}
\theoremstyle{plain}
\newtheorem{theorem}{Theorem}[section]

\newtheorem{lemma}[theorem]{Lemma}
\newtheorem{corollary}[theorem]{Corollary}
\theoremstyle{definition}
\newtheorem{definition}[theorem]{Definition}

\theoremstyle{remark}

\usepackage{xcolor}
\definecolor{darkgreen}{RGB}{0,128,0}

\usepackage[disable,textsize=tiny]{todonotes}

\icmltitlerunning{Fast Inference from Transformers via Speculative Decoding}

\begin{document}

\twocolumn[
\icmltitle{Fast Inference from Transformers via Speculative Decoding}



\icmlsetsymbol{equal}{*}

\begin{icmlauthorlist}
\icmlauthor{Yaniv Leviathan}{equal,google}
\icmlauthor{Matan Kalman}{equal,google}
\icmlauthor{Yossi Matias}{google}
\end{icmlauthorlist}

\icmlaffiliation{google}{Google Research, Mountain View, CA, USA}

\icmlcorrespondingauthor{Yaniv Leviathan}{leviathan@google.com}

\icmlkeywords{Machine Learning, Transformer, Inference, Decoding}

\vskip 0.3in
]



\printAffiliationsAndNotice{\icmlEqualContribution} 

\begin{abstract}
Inference from large autoregressive models like Transformers is slow - decoding $K$ tokens takes $K$ serial runs of the model.
In this work we introduce \emph{speculative decoding} - an algorithm to sample from autoregressive models faster \emph{without any changes to the outputs},
by computing several tokens in parallel.
At the heart of our approach lie the observations that (1) hard language-modeling tasks often include easier subtasks that can be approximated well by more efficient
models, and (2) using speculative execution and a novel sampling method, we can make exact decoding from the large models faster, by running them in parallel on the outputs of the approximation models, potentially generating several tokens concurrently, and without changing the distribution.
Our method can accelerate existing off-the-shelf models without retraining or architecture changes.
We demonstrate it on T5-XXL and show a \textbf{2X-3X} acceleration compared to the standard T5X implementation, with identical outputs.
\end{abstract}

\begin{figure*}
\includegraphics[width=1\linewidth]{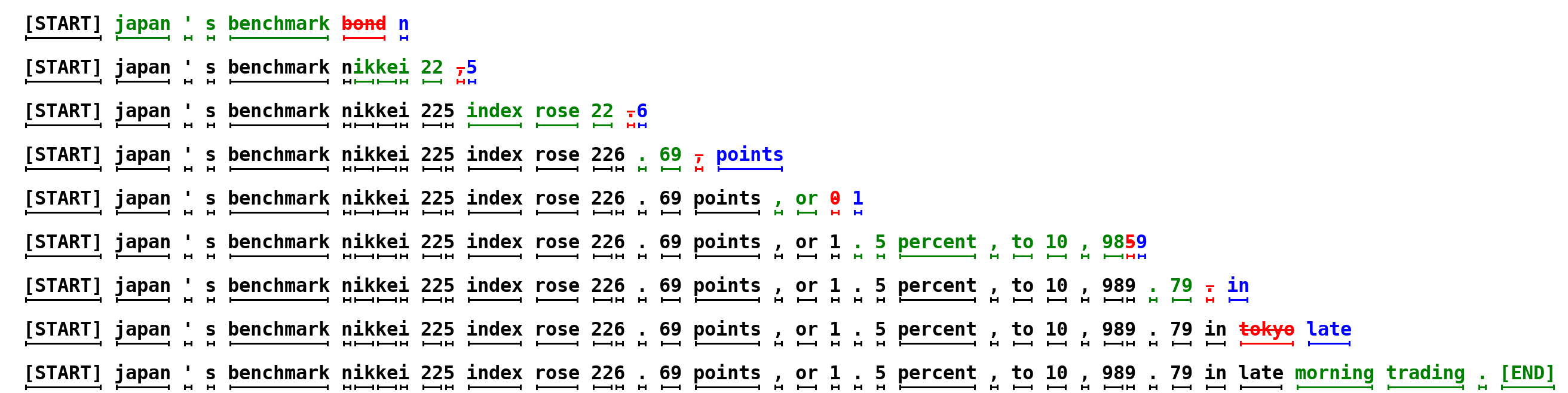}
\caption[test]{Our technique illustrated in the case of unconditional language modeling. Each line represents one iteration of the algorithm. The {\fontfamily{qcr}\selectfont {\color{darkgreen}\underline{\textbf{green}}}} tokens are the suggestions made by the approximation model (here, a GPT-like Transformer decoder with 6M parameters trained on lm1b with 8k tokens) that the target model (here, a GPT-like Transformer decoder with 97M parameters in the same setting) accepted, while the {\fontfamily{qcr}\selectfont {\color{red}\underline{\textbf{red}}}} and {\fontfamily{qcr}\selectfont {\color{blue}\underline{\textbf{blue}}}} tokens are the rejected suggestions and their corrections, respectively. For example, in the first line the target model was run only once, and 5 tokens were generated.}
\centering
\label{fig:autocomplete}
\end{figure*}

\section{Introduction}
\label{sec:introduction}

Large autoregressive models, notably large Transformers \citep{transformer}, are much more capable than smaller models, as is evidenced countless times in recent years e.g., in the text or image domains, like GPT-3 \citep{gpt3}, LaMDA \citep{lamda}, Parti \citep{parti}, and PaLM \citep{palm}. Unfortunately, a single decode step from these larger models is significantly slower than a step from their smaller counterparts, and making things worse, these steps are done serially - decoding $K$ tokens takes $K$ serial runs of the model.

Given the importance of large autoregressive models and specifically large Transformers, several approaches were developed to make inference from them faster.
Some approaches aim to reduce the inference cost for \emph{all} inputs equally
\citep[e.g.][]{Distilling_the_Knowledge_in_a_Neural_Network, Sparse_is_Enough_in_Scaling_Transformers, Quantized_Neural_Networks, Primer, multi-query-attn}.
Other approaches stem from the observation that not all inference steps are born alike - some require a very large model, while others can be approximated well by more efficient models. These \emph{adaptive computation} methods
\citep[e.g.][]{Dynamic_Neural_Networks_Survey, Adaptive_Attention_Span_in_Transformers, Confident_Adaptive_Transformers, early_exits, Controlling_Computation_versus_Quality_Sequence_Models, depth_adaptive_transformer, Matching_Model_and_Instance_Complexities}
aim to use less compute resources for easier inference steps.
While many of these solutions have proven extremely effective in practice, they usually require changing the model architecture, changing the training-procedure and re-training the models, and don't maintain identical outputs.

The key observation above, that some inference steps are ``harder'' and some are ``easier'', is also a key motivator for our work.
We additionally observe that inference from large models is often not bottlenecked on arithmetic operations, but rather on memory bandwidth and communication, so additional computation resources might be available.
Therefore we suggest increasing concurrency as a complementary approach to using an adaptive amount of computation. Specifically, we are able to accelerate inference without changing the model architectures, without changing the training-procedures or needing to re-train the models, and without changing the model output distribution. This is accomplished via \emph{speculative execution}.

Speculative execution \cite{speculative_computation, HennessyPatterson12} is an optimization technique, common in processors, where a task is performed in parallel to verifying if it's actually needed - the payoff being increased concurrency. A well-known example of speculative execution is branch prediction.
For speculative execution to be effective, we need an efficient mechanism to suggest tasks to execute that are likely to be needed.
In this work, we generalize speculative execution to the stochastic setting - where a task \emph{might be} needed with some probability.
Applying this to decoding from autoregressive models like Transformers, we sample generations from more efficient \emph{approximation models} as speculative prefixes for the slower \emph{target models}.
With a novel sampling method, \emph{speculative sampling}, we maximize the probability of these speculative tasks to be accepted, while guaranteeing that the outputs from our system have the same distribution as those from the target model alone.
For example, the sentence in \cref{fig:autocomplete}, consisting of 38 tokens, was generated by our method with only 9 serial runs of a larger target model (97M parameters) thanks to a smaller and more efficient approximation model (6M parameters),
while the probability of generating it is unchanged.

We analyze our method in a variety of tasks and model sizes: unconditional generation from a 97M parameter GPT-like model trained on lm1b, English to German translation and news article summarization with an 11B parameters T5-XXL model, and a dialog task with a 137B parameter LaMDA model. We implement our method and compare actual walltimes for T5-XXL to those of the robust T5X implementation \cite{t5x}, showing an out-of-the-box latency improvement of \textbf{2X-3X}, without any change to the outputs (\cref{sec:experiments}).

Our method is easy to employ in actual production settings, doesn't require training new models, and doesn't change the outputs.
Therefore, in common situations where memory bandwidth is the bottleneck, and  compute resources are available, it may be a good default to accelerate sampling from autoregressive models like Transformers.

To summarize, our main contributions are: (1) A generalization of speculative execution to the stochastic setting, with a novel sampling method we call \emph{speculative sampling}, and (2) A decoding mechanism we call \emph{speculative decoding} that can accelerate decoding from autoregressive models, without any change to the model architectures, training regimes and output distributions.

\section{Speculative Decoding}
\subsection{Overview}

Let $M_p$ be the target model, inference from which we're trying to accelerate, and $p(x_t|x_{<t})$ the distribution we get from the model for a prefix $x_{<t}$. Let $M_q$ be a more efficient approximation model for the same task, and denote by $q(x_t|x_{<t})$ the distribution we get from the model for a prefix $x_{<t}$\footnote{
We'll use $p(x)$ to denote $p(x_t|x_{<t})$ whenever the prefix $x_{<t}$ is clear from the context, and similarly for $q(x)$.}. The core idea is to (1) use the more efficient model $M_q$ to generate $\gamma \in \mathbb{Z}^+$ completions (see \cref{sec:optimal_gamma} for how to optimally choose this parameter), then (2) use the target model $M_p$ to evaluate all of the guesses and their respective probabilities from $M_q$ \emph{in parallel}, accepting all those that \emph{can} lead to an identical distribution, and (3) sampling an additional token from an adjusted distribution to fix the first one that was rejected, or to add an additional one if they are all accepted. That way, each parallel run of the target model $M_p$ will produce at least one new token (so the number of serial runs of the target model can never, even in the worst case, be larger than the simple autoregressive method), but it can potentially generate many new tokens, up to $\gamma + 1$, depending on how well $M_q$ approximates $M_p$.

\subsection{Standardized Sampling}
\label{sec:standardized_sampling}
First, note that while there are many methods and parameters of sampling, like argmax, top-k, nucleus, and setting a temperature, and popular implementations usually treat them differently at the logits level, they can all easily be cast into standard sampling from an adjusted probability distribution. For example, argmax sampling is equivalent to zeroing out non-max elements of the distribution and normalizing. We can therefore only deal with standard sampling from a probability distribution, and cast all of the other types of sampling into that framework. Going forward we'll assume that $p(x)$ and $q(x)$ are the distributions from $M_p$ and $M_q$ respectively, adjusted for the sampling method.

\subsection{Speculative Sampling}
\label{sec:speculative_sampling}
To sample $x \sim p(x)$, we instead sample $x \sim q(x)$, keeping it if $q(x) \le p(x)$, and in case $q(x) > p(x)$ we reject the sample with probability $1 - \frac{p(x)}{q(x)}$ and sample $x$ again from an adjusted distribution $p'(x) = norm(max(0, p(x) - q(x)))$ instead. It’s easy to show (see \cref{appendix:correctness}) that for any distributions $p(x)$ and $q(x)$, and $x$ sampled in this way, indeed $x \sim p(x)$.

Given the distribution $q(x)$ obtained from running $M_q$ on a conditioning $prefix$, we can sample a token $x_1 \sim q(x)$. We then calculate the distribution $p(x)$ by running $M_p$ on $prefix$ while in parallel speculatively calculating the distribution of the next token $x_2$ by running $M_p$ on $prefix + [x_1]$. Once both computations complete, we proceed as per above: If $x_1$ is rejected, we discard the computation of $x_2$ and re-sample $x_1$ from an adjusted distribution, and if $x_1$ is accepted, we keep both tokens. \cref{alg:sample_with_approx} generalizes this idea to sample between 1 and $\gamma + 1$ tokens at once.

\newcommand{\COMMENTLLAMA}[1]{{\color{darkgreen} $\triangleright$ {#1}}}

\begin{algorithm}
  \caption{SpeculativeDecodingStep}
  \label{alg:sample_with_approx}
\begin{algorithmic}
  \STATE {\bfseries Inputs:} $M_p, M_q, prefix$.
  \STATE \COMMENTLLAMA{Sample $\gamma$ guesses $x_{1,\ldots,\gamma}$ from $M_q$ autoregressively.}
  \FOR{$i=1$ {\bfseries to} $\gamma$}
    \STATE $q_i(x) \gets M_q(prefix + [x_1, \ldots, x_{i-1}])$
    \STATE $x_i \sim q_i(x)$
  \ENDFOR
  \STATE \COMMENTLLAMA{Run $M_p$ in parallel.}
  \STATE $p_1(x), \ldots, p_{\gamma + 1}(x) \gets$
  \STATE \quad \quad $M_p(prefix), \ldots, M_p(prefix + [x_1, \ldots, x_{\gamma}])$
  \STATE \COMMENTLLAMA{Determine the number of accepted guesses $n$.}
  \STATE $r_1 \sim U(0, 1), \dots, r_\gamma \sim U(0, 1)$
  \STATE $n \gets \min(\{ i - 1 \mid 1 \le i \le \gamma, r_i > \frac{p_i(x)}{q_i(x)} \} \cup \{ \gamma \})$
  \STATE \COMMENTLLAMA{Adjust the distribution from $M_p$ if needed.}
  \STATE $p'(x) \gets p_{n+1}(x)$
  \IF{$n < \gamma$}
    \STATE $p'(x) \gets norm(max(0, p_{n+1}(x) - q_{n+1}(x)))$
  \ENDIF
  \STATE \COMMENTLLAMA{Return one token from $M_p$, and $n$ tokens from $M_q$.}
  \STATE $t \sim p'(x)$
  \STATE {\bfseries return} $prefix + [x_1, \ldots, x_{n}, t]$
\end{algorithmic}
\end{algorithm}

\section{Analysis}

\subsection{Number of Generated Tokens}
\label{sec:num_generated_tokens}
Let's analyze the reduction factor in the number of serial calls to the target model, or equivalently, the expected number of tokens produced by a single run of \cref{alg:sample_with_approx}.

\begin{definition}
\label{def:beta_prefix}
The \emph{acceptance rate $\beta_{x_{<t}}$}, given a prefix $x_{<t}$, is the probability of accepting $x_t \sim q(x_t|x_{<t})$ by speculative sampling, as per \cref{sec:speculative_sampling}\footnote{As before, we'll omit the $x_{<t}$ subscript whenever the prefix is clear from the context.}.
\end{definition}

$E(\beta)$ is then a natural measure of how well $M_q$ approximates $M_p$. If we make the simplifying assumption that the $\beta$s are i.i.d., and denote $\alpha = E(\beta)$, then the number of tokens produced by a single run of \cref{alg:sample_with_approx} is a capped geometric variable, with success probability $1 - \alpha$ and cap $\gamma + 1$, and the expected number of tokens generated by \cref{alg:sample_with_approx} satisfies \cref{eq:expected_num_tokens}. See \cref{fig:alpha_v_tokens}.

\begin{equation}
\label{eq:expected_num_tokens}
E(\#\ generated\ tokens) = \frac{1 - \alpha^{\gamma + 1}}{1 - \alpha}
\end{equation}

\begin{figure}[ht]
\begin{center}
\centerline{\includegraphics[width=\columnwidth]{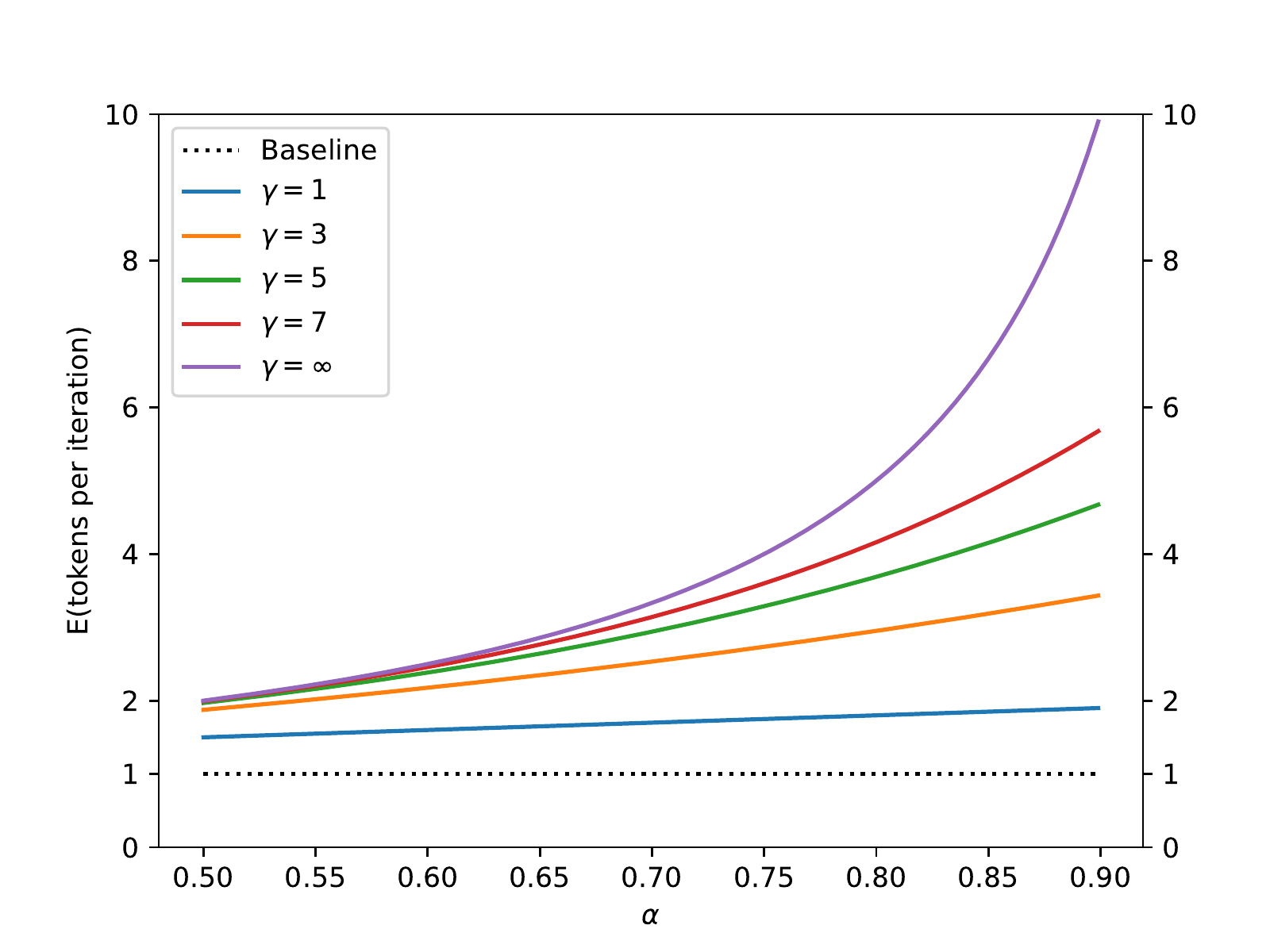}}
\caption{The expected number of tokens generated by \cref{alg:sample_with_approx} as a function of $\alpha$ for various values of $\gamma$.}
\label{fig:alpha_v_tokens}
\end{center}
\vskip -0.2in
\end{figure}

\subsection{Calculating $\alpha$}

We'll now derive a simple formula for calculating $\alpha$ given a prefix and the two models $M_p$ and $M_q$. We start by defining a natural divergence $D_{LK}$:

\begin{definition}
\label{def:dlk}
$D_{LK}(p, q) = \sum_{x}|p(x) - M(x)| = \sum_{x}|q(x) - M(x)|$ where $M(x) = \frac{p(x) + q(x)}{2}$.
\end{definition}

\begin{lemma}
\label{lemma:dlk}
$D_{LK}(p, q) = 1 - \sum_x\min(p(x), q(x))$
\end{lemma}

\begin{proof}
$D_{LK}(p, q) = \sum_{x}|p(x) - M(x)| = \sum_x\frac{|p-q|}{2}= 1 - \sum_x\frac{p + q - |p - q|}{2} = 1 - \sum_x\min(p(x), q(x))$
\end{proof}

From \cref{lemma:dlk} we immediately get the following results:

\begin{corollary}
\label{cor:dlk_properties}
$D_{LK}(p, q)\ \text{is a symmetric divergence in}\ [0, 1].\\
D_{LK}(p, q) = 0 \iff p = q. \\
D_{LK}(p, q) = 1 \iff \text{p and q have disjoint support}.$
\end{corollary}

\begin{theorem}
\label{thm:dlk_beta}
$\beta = 1 - D_{LK}(p, q)$
\end{theorem}
\begin{proof}
$\beta = E_{x \sim q(x)}\left\lbrace
\begin{tabular}{@{}p{0.5cm}p{1.8cm}} 
$1$ & $q(x) \le p(x)$ \\
$\frac{p(x)}{q(x)}$ & $q(x) > p(x)$ \\
\end{tabular}\right. =
E_{x \sim q(x)}\min(1, \frac{p(x)}{q(x)}) =
\sum_x\min(p(x), q(x))$
\end{proof}

Finally we get:

\begin{corollary}
\label{thm:dlk_alpha}
$\alpha = 1 - E(D_{LK}(p, q)) = E(\min(p, q))$
\end{corollary}

See \cref{tbl:empirical alphas} for empirically observed $\alpha$ values in our experiments.

\subsection{Walltime Improvement}
\label{sec:latency_imp}

We've shown that with the i.i.d. assumption our algorithm reduces the number of calls to the target model by a factor of $\frac{1 - \alpha^{\gamma + 1}}{1 - \alpha}$.
Note that speculative execution in general, and our algorithm in particular, assume that we have enough compute resources to support the increased concurrency (\cref{sec:arithmetic_ops}). For the walltime anaylsis, we'll assume that we can run $\gamma + 1$ concurrent evaluations of $M_p$ in parallel without increasing the walltime.
To get the total walltime improvement, we now consider the cost of running the approximation model $M_q$.

\begin{definition}
\label{def:c}
Let $c$, the \emph{cost coefficient}, be the ratio between the time for a single run of $M_q$ and the time for a single run of $M_p$.
\end{definition}

Note that unlike $\alpha$ which is an intrinsic property of the models and the task, the value of $c$ depends on the hardware configuration and software implementation details. In our experiments where $M_q$ is typically a couple of orders of magnitude smaller than $M_p$, $c$ was always less than $0.05$ and often negligibly close to 0.

\begin{theorem}
\label{thm:total_walltime}
The expected improvement factor in total walltime by \cref{alg:sample_with_approx} is $\frac{1 - \alpha^{\gamma + 1}}{(1-\alpha)({\gamma}c + 1)}$.

\end{theorem}
\begin{proof}
Denote the cost of running a single step of $M_p$ by $T$. Now, each run of \cref{alg:sample_with_approx} costs $Tc\gamma + T$ (for running the approximation model $M_q$ $\gamma$ times and running $M_p$ once) and according to \cref{eq:expected_num_tokens} produces $\frac{1 - \alpha^{\gamma + 1}}{1 - \alpha}$ tokens on average.
So the overall expected cost for producing a token with \cref{alg:sample_with_approx} is $\frac{(c\gamma + 1)(1-\alpha)}{1 - \alpha^{\gamma + 1}}T$.
Since the cost of producing a single token with the standard decoding algorithm is $T$, we get the desired result.
\end{proof}

Note that \cref{thm:total_walltime} assumes long enough generations (for example, since we run $M_p$ at least once, the improvement factor is capped by the number of generated tokens).

\begin{corollary}
\label{lemma:condition_for_improvement}
If $\alpha > c$, there exists $\gamma$ for which we'll get an improvement, and the improvement factor will be at least $\frac{1 + \alpha}{1 + c}$.
\end{corollary}
\begin{proof}
If we get an improvement for $\gamma$, we’d also get an improvement for any $0 < \gamma^* < \gamma$, so for our method to yield an improvement, we can evaluate \cref{thm:total_walltime} for $\gamma=1$, yielding $\frac{1 - \alpha^2}{(1-\alpha)(c + 1)} = \frac{1 + \alpha}{1 + c}$.
\end{proof}

\subsection{Number of Arithmetic Operations}
\label{sec:arithmetic_ops}

\cref{alg:sample_with_approx} does $\gamma + 1$ runs of $M_p$ in parallel, so the number of \emph{concurrent} arithmetic operations grows by a factor of $\gamma + 1$.
Now, since \cref{alg:sample_with_approx} produces at most $\gamma + 1$ tokens per run, the \emph{total} number of arithmetic operations might be higher than that of the standard decoding algorithm.
When we accept the sample from $M_q$ the increased concurrency is ``free'' and the total number of operations isn't increased\footnote{Neglecting the cost of $M_q$.}. When we reject a guess though, computation is wasted.
Let's now analyze the effect of our method on the total number of arithmetic operations.

\begin{definition}
Let $\hat{c}$ be the ratio of arithmetic operations per token of the approximation model $M_q$ to that of the target model $M_p$.
\end{definition}

\begin{theorem}
\label{thm:num_ops}
The expected factor of increase in the number of total operations of \cref{alg:sample_with_approx} is $\frac{(1-\alpha)({\gamma}\hat{c} + \gamma + 1)}{1 - \alpha^{\gamma + 1}}$.
\end{theorem}

\begin{proof}
Denote by $\hat{T}$ the number of arithmetic operations done by a standard decoding baseline per token, i.e. the number of operations of a single run of $M_p$. Then a single iteration of \cref{alg:sample_with_approx} costs $\hat{T}\hat{c}\gamma + \hat{T}(\gamma + 1)$ operations (for $\gamma$ runs of $M_q$ and $\gamma + 1$ parallel runs of $M_p$). Dividing by the expected number of tokens produced by \cref{alg:sample_with_approx}, i.e. \cref{eq:expected_num_tokens}, and by $\hat{T}$, we get the desired result.
\end{proof}

If $\alpha$ is low, the increase in the number of arithmetic operations is high, and vice-versa.
Note that for Transformer decoders, the total number of arithmetic operations by \cref{alg:sample_with_approx} (not counting runs of $M_q$) \emph{can be bounded from above by a single run of the same-size Transformer encoder}.

Unlike the total number of arithmetic operations, the total number of memory accesses can go down with our method. Specifically, the target model's weights and KV cache can be read once per execution of \cref{alg:sample_with_approx}, so the number of memory accesses for reading them shrinks by a factor of $\frac{1 - \alpha^{\gamma + 1}}{1 - \alpha}$, according to \cref{eq:expected_num_tokens}.

\begin{figure}[ht]
\vskip 0.2in
\begin{center}
\centerline{\includegraphics[width=\columnwidth]{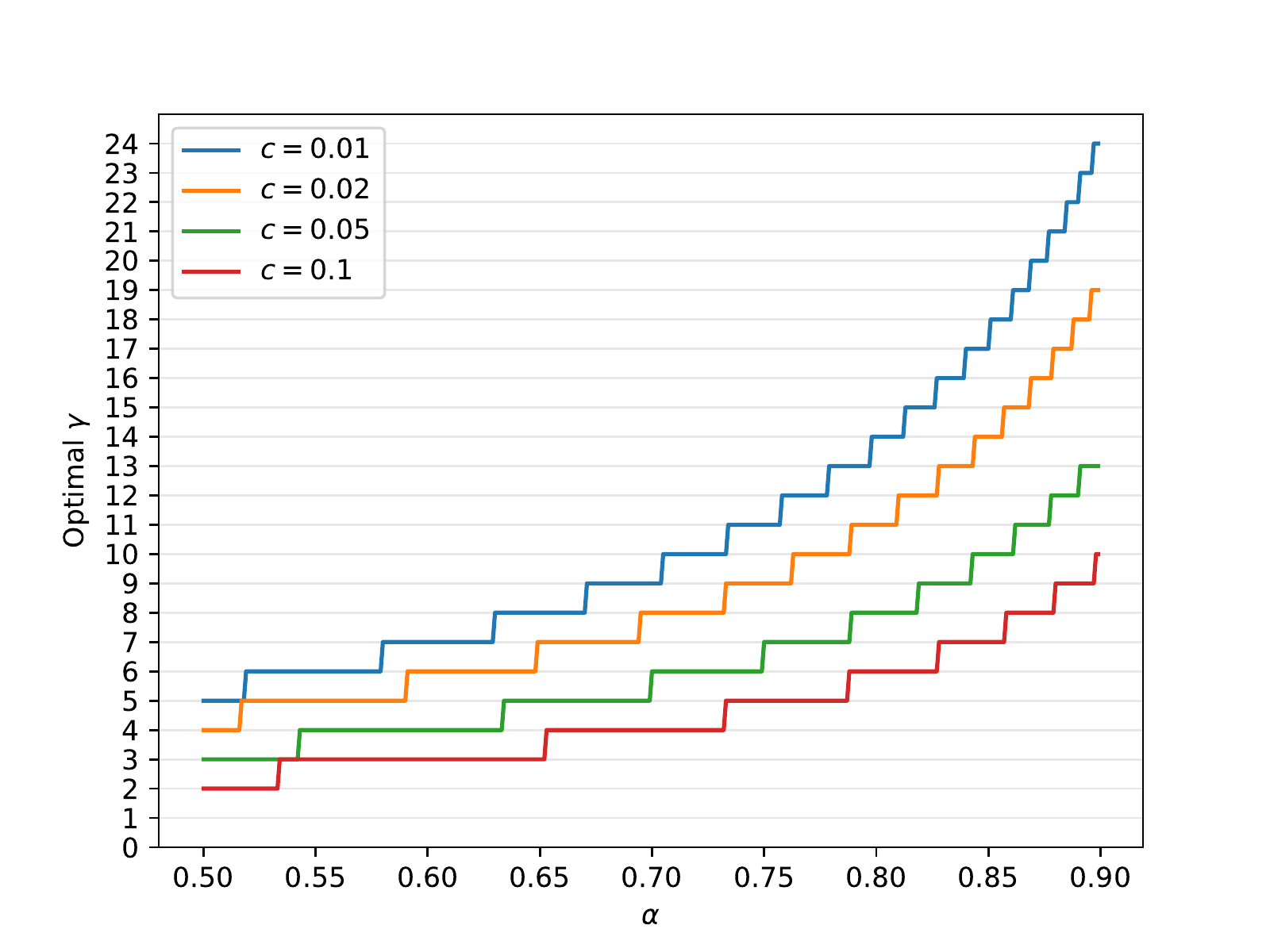}}
\caption{The optimal $\gamma$ as a function of $\alpha$ for various values of $c$.}
\label{fig:alpha_v_optimal_gamma}
\end{center}
\vskip -0.2in
\end{figure}

\subsection{Choosing \texorpdfstring{$\gamma$}{gamma}}
\label{sec:optimal_gamma}
Given $c$ and $\alpha$ and assuming enough compute resources (see \cref{sec:arithmetic_ops}), the optimal $\gamma$ is the one maximizing the walltime improvement equation (\cref{thm:total_walltime}): $\frac{1 - \alpha^{\gamma + 1}}{(1-\alpha)({\gamma}c + 1)}$. Since $\gamma$ is an integer, it can be easily found numerically, see \cref{fig:alpha_v_optimal_gamma}.

\cref{table:non optimal gamma} and \cref{fig:arithm_ops} illustrate the trade-off between inference speed and the total number of arithmetic operations for various values of $\alpha$ and $\gamma$, assuming $c = \hat{c} = 0$.
\cref{fig:xprof} shows a simplified trace diagram.

\begin{table}[ht]
\caption{The total number of arithmetic operations and the inference speed vs the baseline, for various values of $\gamma$ and $\alpha$, assuming $c = \hat{c} = 0$.}
\label{table:non optimal gamma}
\vskip 0.15in
\begin{center}
\begin{small}
\begin{sc}
\begin{tabular}{lcccr}
\toprule
$\alpha$ & $\gamma$ & Operations & Speed \\
\midrule
0.6 & 2 & 1.53X & 1.96X \\
0.7 & 3 & 1.58X & 2.53X \\
0.8 & 2 & 1.23X & 2.44X \\
0.8 & 5 & 1.63X & 3.69X \\
0.9 & 2 & 1.11X & 2.71X \\
0.9 & 10 & 1.60X & 6.86X \\
\bottomrule
\end{tabular}
\end{sc}
\end{small}
\end{center}
\vskip -0.1in
\end{table}

\begin{figure}[ht]
\vskip 0.2in
\begin{center}
\centerline{\includegraphics[width=\columnwidth]{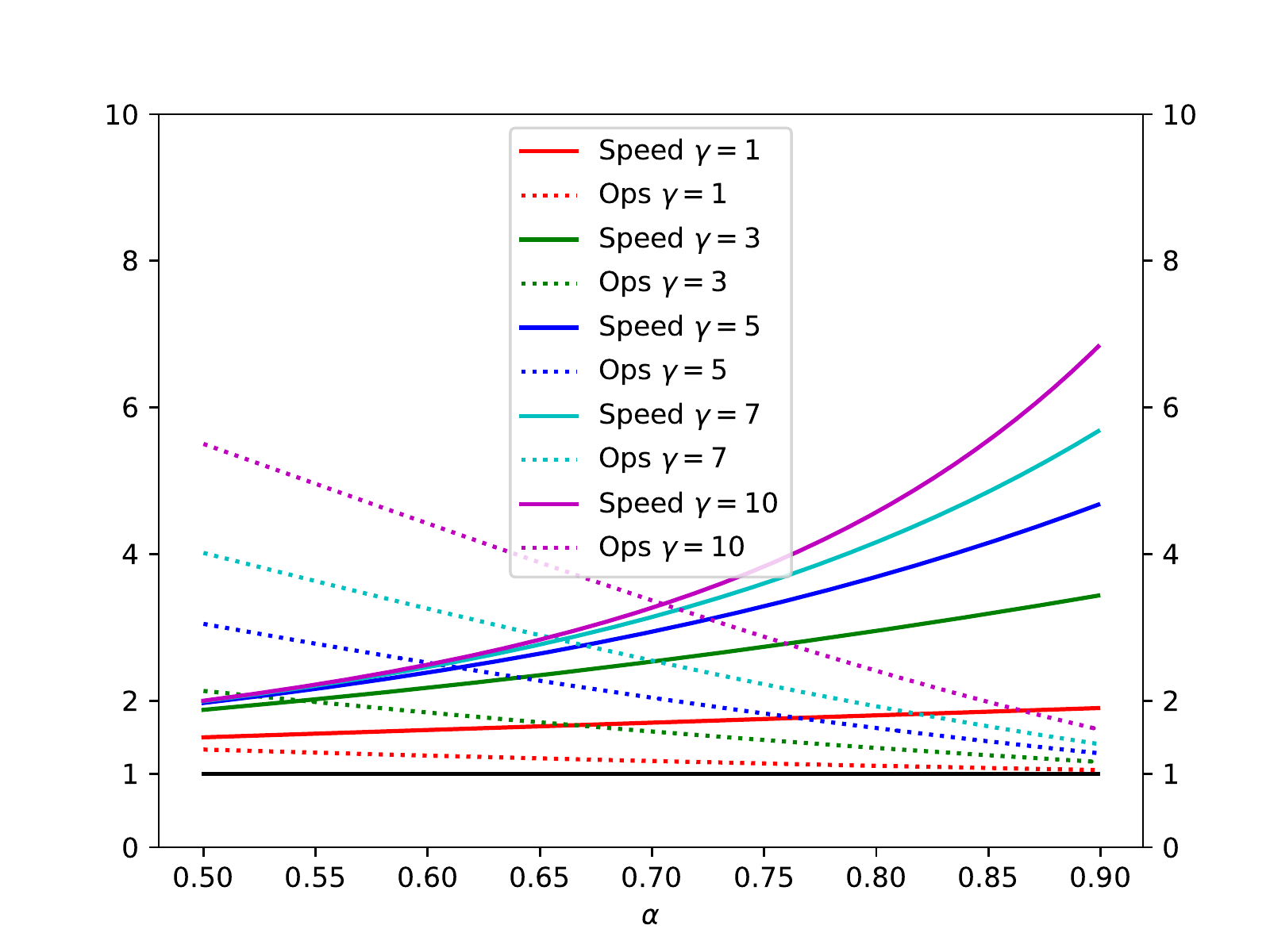}}
\caption{The speedup factor and the increase in number of arithmetic operations as a function of $\alpha$ for various values of $\gamma$.}
\label{fig:arithm_ops}
\end{center}
\vskip -0.2in
\end{figure}

\begin{figure*}
\includegraphics[width=1\linewidth]{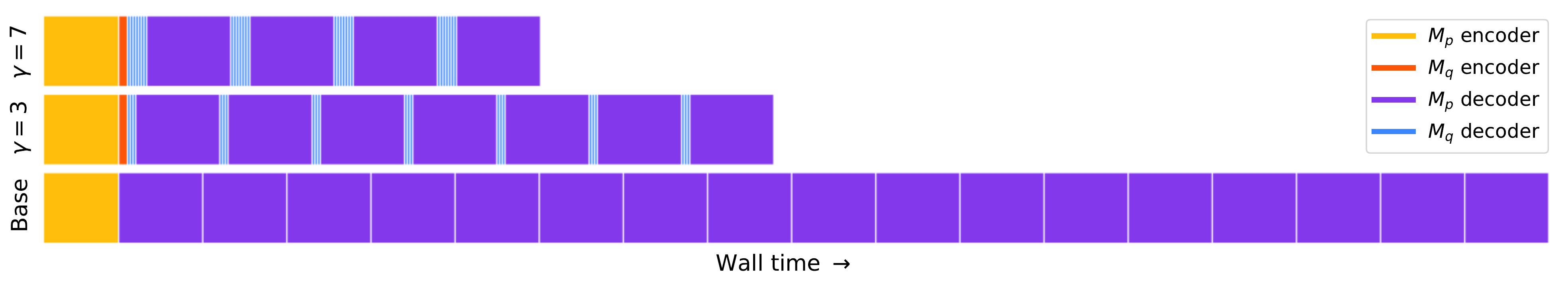}
\caption[xprof]{A simplified trace diagram for a full encoder-decoder Transformer stack.
The top row shows speculative decoding with $\gamma=7$ so each of the calls to $M_p$ (the purple blocks) is preceded by 7 calls to $M_q$ (the blue blocks). The yellow block on the left is the call to the encoder for $M_p$ and the orange block is the call to the encoder for $M_q$. Likewise the middle row shows speculative decoding with $\gamma=3$, and the bottom row shows standard decoding.}
\centering
\label{fig:xprof}
\end{figure*}

Instead of picking a single value for $\gamma$ based on $\alpha$, since the $\beta$s aren’t constant,
we could get further improvement by predicting the value of $\beta$ and accordingly varying the value of $\gamma$ during the run of \cref{alg:sample_with_approx}. To get an upper bound on the additional improvement factor, assume we had an oracle for $\gamma$. We would then have $E(\#\ generated\ tokens) = \frac{1}{1-\alpha}$.
For typical values of $c$ and $\alpha$, and assuming unbounded compute resources, the enhanced walltime improvement factor can be up to $\sim$60\% higher than the improvement factor with a fixed $\gamma$. We leave exploring this for future work\footnote{The above bound assumes that we still run $M_p$ to verify the oracle's predictions. If we skip those verifications the bound doesn't hold and we would get a substantial additional improvement.}.

\subsection{Approximation Models}
\label{sec:approximation_models}
Speculative sampling, and therefore speculative decoding, guarantee an identical output distribution for any choice of approximation model $M_q$ without restriction (see \cref{appendix:correctness}).
In our experiments, we mostly tested existing off-the-shelf smaller Transformers as the approximation models.
Further, we only tested approximation models of the same architecture as the target models $M_p$ and using the same probability standardization.
In this setup, choosing $M_q$ to be around two orders of magnitude smaller than $M_p$ usually performed best, balancing $\alpha$ and $c$ (\cref{thm:total_walltime}).

Another type of approximation models, \emph{negligible-cost models}, are those for which $c \approx 0$, i.e. approximation models with a negligible cost relative to the target model. In this case, we get an expected walltime improvement of $\frac{1 - \alpha^{\gamma + 1}}{1 - \alpha}$, which is bounded from above by $\frac{1}{1 - \alpha}$ (we approach equality if $\gamma$ is large).
One interesting type of negligible-cost approximation models are n-gram models, where the evaluation amounts to a table lookup. Interestingly, in empirical tests (\cref{sec:empirical_alphas}) we get non zero $\alpha$s even for these trivial n-gram models. For example, for the English-German translation task, with $M_p$ being T5-XXL 11B and $M_q$ being a trivial bigram model, we get $\alpha \approx 0.2$ which leads to an inference speed improvement factor of $1.25$X with $\gamma=3$.

Other simple heuristics can be used as negligible-cost approximation models.
For example, in cases where long sequences are likely to repeat, such as for summarization tasks or chat-like interfaces \footnote{E.g. where a user and a language model iterate on content, like text or code (``can you rewrite this story but change the ending'', ``can you make this function also do X'').}, an approximation model that simply copies tokens from the context in case we find a matching prefix, might yield high values of $\alpha$.
These parameter-less approximation models, have the additional advantage of being even simpler to deploy from a production standpoint.

Another type of approximation models that can be used by speculative decoding are non-autoregressive models, like those from \citep{Blockwise_Parallel_Decoding}. Then, instead of the autogreressive loop in \cref{alg:sample_with_approx} we'd just call the non-autoregressive model once.

A final example, interesting mostly from a theoretical perspective, is an approximation model which chooses tokens at random, which guarantees some improvement (although very small) for all models $M_p$.

\vspace{1.04cm}

\section{Experiments}
\label{sec:experiments}

\subsection{Empirical Walltime Improvement}
\label{sec:empirical-t5-time}

We implement our algorithm and compare it to the implementation in the T5X codebase for accelerating T5-XXL.

\paragraph{Setup} We test a standard encoder-decoder T5 version 1.1 model \citep{t5} on two tasks from the T5 paper: (1) English to German translation fine tuned on WMT EnDe, and (2) Text summarization fine tuned on CCN/DM. For both tasks, we use T5-XXL (11B) for $M_p$. For the approximation model $M_q$ we test several existing configurations, namely T5-large (800M), T5-base (250M), and T5-small (77M) \citep{t5}. We use existing checkpoints for all models. We measure walltime improvements with a batch size of 1 on a single TPU-v4 for both argmax sampling (temp=0) and standard sampling (temp=1).

\paragraph{Results}

\cref{tbl:t5 model results} shows the empirical results from our method. We see that T5-small (77M), with a good balance of $c$ and $\alpha$, provides the highest speedup out of the tested approximation models. As expected we see that $\alpha$ increases with the size of the approximation model. Interestingly, $\alpha$ and walltime improvement are higher for argmax sampling (temp=0). We observe speedups of 2.6X (temp=1) and 3.4X (temp=0) on the translation task and slightly lower speedups of 2.3X (temp=1) and 3.1X (temp=0) for the summarization task. These empirical results match well with the theoretical predictions, with some variance due to implementation details (see \cref{sec:theory_vs_emp}).

\begin{table}[ht]
  \caption{Empirical results for speeding up inference from a T5-XXL 11B model.}
  \label{tbl:t5 model results}
\vskip 0.15in
\begin{center}
\begin{small}
\begin{sc}
\begin{tabular}{llcccrr}
\toprule
Task & $M_q$ & Temp & $\gamma$ & $\alpha$ & Speed\\
\midrule
EnDe & T5-small $\bigstar$     & 0 & 7 &  0.75  &  \textbf{3.4X}\\
EnDe & T5-base      & 0 & 7 &  0.8  &  2.8X \\
EnDe & T5-large      & 0 & 7 &  0.82  &  1.7X\\
EnDe & T5-small $\bigstar$     & 1 & 7 &  0.62  &  \textbf{2.6X}\\
EnDe & T5-base      & 1 & 5 &  0.68  &  2.4X \\
EnDe & T5-large      & 1 & 3 &  0.71  &  1.4X \\
\midrule
CNNDM & T5-small $\bigstar$      & 0 & 5 &  0.65  &  \textbf{3.1X} \\
CNNDM & T5-base      & 0 & 5 &  0.73  &  3.0X \\
CNNDM & T5-large      & 0 & 3 &  0.74  &  2.2X\\
CNNDM & T5-small $\bigstar$      & 1 & 5 &  0.53  &  \textbf{2.3X}\\
CNNDM & T5-base      & 1 & 3 &  0.55  &  2.2X \\
CNNDM & T5-large      & 1 & 3 &  0.56  &  1.7X \\
\bottomrule
\end{tabular}
\end{sc}
\end{small}
\end{center}
\vskip -0.1in
\end{table}

\subsection{Empirical $\alpha$ Values}
\label{sec:empirical_alphas}

While we only implemented our method for T5, we measured $\alpha$ values for various tasks, sampling methods, target models $M_p$, and approximation models $M_q$. Specifically, we evaluated the expectation from \cref{thm:dlk_alpha} on 10K tokens generated by $M_p$, for each of the settings below.

\paragraph{GPT-like (97M params)}We test a decoder-only Transformer model on unconditional language generation, trained on lm1b \citep{lm1b}. The model here is a GPT-like Transformer decoder with Gelu activations \citep{gelu}. For $M_q$ we experimented with a Transformer decoder model with 6M parameters: dim 256, dim feed-forward 1024, 2 layers, 4 attention heads, as well as simple unigram and bigram models. $M_p$ has 97M parameters: dim 768, dim feed-forward 3072, 12 layers, 12 attention heads. We used Bert tokenization \citep{bert} with 8k tokens for all models.

\paragraph{LaMDA (137B params)} We tested a decoder only LaMDA model on a dialog task \citep{lamda}.
We used existing checkpoints from LaMDA 137B as $M_p$ and LaMDA 8B, LaMDA 2B, and LaMDA 100M for $M_q$.

See \cref{sec:empirical-t5-time} for the setup of the T5-XXL (11B params) model.

\cref{tbl:empirical alphas} summarizes the $\alpha$ values for the tested cases. We observe that approximation models that are a couple of orders of magnitude smaller than the target model tend to produce $\alpha$ values between 0.5 and 0.9. Interestingly, we also note that for all models, the sharper the adjusted distribution, the higher the $\alpha$ values. Finally, we note that even trivial unigram and bigram approximations yield non negligible $\alpha$ values. For example, for the case of English to German translation, the bigram model has an $\alpha$ value of 0.2, and since $c=0$ in this case, yields a 1.25X speed improvement, which is surprisingly high for this trivial approximation model (but is still lower than the speedup we get from using T5-small as the approximation model).

\stepcounter{footnote}

\begin{table}[ht]
  \caption{Empirical $\alpha$ values for various target models $M_p$, approximation models $M_q$, and sampling settings. T=0 and T=1 denote argmax and standard sampling respectively\footnotemark[\value{footnote}].}
  \label{tbl:empirical alphas}
\vskip 0.15in
\begin{center}
\begin{small}
\begin{sc}
\begin{tabular}{lllccr}
\toprule
    $M_p$ & $M_q$ & Smpl & $\alpha$ \\
    
    \midrule
    GPT-like (97M) & Unigram      & t=0    & 0.03\\
    GPT-like (97M) & Bigram       & t=0    &  0.05\\
    GPT-like (97M) & GPT-like (6M)  & t=0    &  0.88\\
    GPT-like (97M) & Unigram      & t=1  &  0.03\\
    GPT-like (97M) & Bigram       & t=1  &  0.05\\
    GPT-like (97M) & GPT-like (6M)  & t=1  &  0.89\\
    \midrule
    \midrule
    
    T5-XXL (EnDe) & Unigram     & t=0 & 0.08 \\
    T5-XXL (EnDe) & Bigram      & t=0 & 0.20 \\
    T5-XXL (EnDe) & T5-small    & t=0 & 0.75\\
    T5-XXL (EnDe) & T5-base     & t=0 & 0.80\\
    T5-XXL (EnDe) & T5-large    & t=0 & 0.82\\
    T5-XXL (EnDe) & Unigram     & t=1 & 0.07 \\
    T5-XXL (EnDe) & Bigram      & t=1 & 0.19 \\
    T5-XXL (EnDe) & T5-small    & t=1 & 0.62\\
    T5-XXL (EnDe) & T5-base     & t=1 & 0.68\\
    T5-XXL (EnDe) & T5-large    & t=1 & 0.71 \\
    \midrule
    \midrule
    
    T5-XXL (CNNDM) & Unigram        & t=0 & 0.13 \\
    T5-XXL (CNNDM) & Bigram         & t=0 & 0.23 \\
    T5-XXL (CNNDM) & T5-small       & t=0 &  0.65\\
    T5-XXL (CNNDM) & T5-base        & t=0 &  0.73\\
    T5-XXL (CNNDM) & T5-large       & t=0 &  0.74\\
    T5-XXL (CNNDM) & Unigram        & t=1 & 0.08 \\
    T5-XXL (CNNDM) & Bigram         & t=1 & 0.16 \\
    T5-XXL (CNNDM) & T5-small      & t=1 & 0.53 \\
    T5-XXL (CNNDM) & T5-base       & t=1  & 0.55 \\
    T5-XXL (CNNDM) & T5-large      & t=1 & 0.56 \\
    \midrule
    \midrule
    
    LaMDA (137B)    & LaMDA (100M) & t=0 & 0.61\\
    LaMDA (137B)    & LaMDA (2B)   & t=0 & 0.71\\
    LaMDA (137B)    & LaMDA (8B)   & t=0 & 0.75\\
    LaMDA (137B)    & LaMDA (100M) & t=1 & 0.57\\
    LaMDA (137B)    & LaMDA (2B)   & t=1 & 0.71\\
    LaMDA (137B)    & LaMDA (8B)   & t=1 & 0.74\\
\bottomrule
\end{tabular}
\end{sc}
\end{small}
\end{center}
\vskip -0.1in
\end{table}

\section{Related work}
The efficiency of inference from large models was studied extensively \citep{The_Efficiency_Misnomer}. Many approaches aim to speed up inference from large models in general, and autoregressive models like Transformers in particular. Numerous techniques try to make inference more efficient for all tokens, e.g. distillation \citep{Distilling_the_Knowledge_in_a_Neural_Network}, sparcification \citep{Sparse_is_Enough_in_Scaling_Transformers}, quantization \citep{Quantized_Neural_Networks}, and architecture modification \citep{Primer, multi-query-attn}. Closer to our approach are adaptive computation methods which adapt the amount of computation to problem difficulty \citep{Dynamic_Neural_Networks_Survey}. Examples include attending to a subset of the inputs \citep{Adaptive_Attention_Span_in_Transformers}, and early exits \citep{Confident_Adaptive_Transformers, early_exits, Controlling_Computation_versus_Quality_Sequence_Models, depth_adaptive_transformer, Matching_Model_and_Instance_Complexities}.
Notably, Wisdom of Committees \citep{Matching_Model_and_Instance_Complexities} leverages off-the-shelf smaller models, but is an adaptive computation approach, and so it uses a heuristic to determine when to stop, losing the guarantee of identical outputs to those of the target models.
In general, adaptive computation methods usually learn, either within the model itself or with an auxiliary model, when a computation shortcut can be taken.
Usually, these methods save on both inference time and arithmetic operations, but require a change of architecture, a change of training procedure and training custom models or re-training of existing models. They usually also change the outputs of the model.
We note that while many of the methods above improve the memory to arithmetic-operations ratio, in cases where the ratio remains high, these methods and our speculative decoding method might be effective in tandem.

Two prior methods leverage speculative execution for speeding up decoding from autoregressive models.
Blockwise Parallel Decoding \citep{Blockwise_Parallel_Decoding} decodes several tokens in parallel, similarly to our work.
However, it only supports greedy decoding (temperature=0) and not the general stochastic setting, it requires additional training of a custom model, and focuses on preserving down-stream task quality, instead of guaranteeing identical outputs.
Shallow Aggressive Decoding (SAD) \citep{sad} also decodes several tokens in parallel, similarly to our work.
Unlike our work, SAD only supports copying the input to the output, and not general approximation models, making it only suitable for the cases where the inputs and outputs are very similar like grammatical error correction. In addition, similarly to Blockwise Parallel Decoding, SAD does not support the general stochastic sampling setting.

After we initially published our work, an independent implementation of speculative decoding \citep{chen2023accelerating} showed similar 2X-2.5X improvements on Chinchilla 70B.

\footnotetext{Note that the outputs from the LaMDA model always go through a $Top_{40}$ filter. This has no effect on argmax, but does have some effect on standard sampling.}

\section{Discussion}

We presented \emph{speculative sampling} which enables efficient \emph{stochastic speculative execution} - i.e. speculative execution in the stochastic setting. We analyzed its impact on decoding from autoregressive models like Transformers via \emph{speculative decoding} and have shown that given enough compute resources, we get meaningful 2X-3X speedups in practice vs T5X, a popular optimized implementation.

One limitation of speculative execution in general, and of speculative decoding in particular, is that latency is improved through increased concurrency at the cost of an increased number of arithmetic operations. Thus, our method is not helpful for configurations where additional computation resources are not available. However, in common cases where additional computation resources are available (e.g. when memory bandwidth is the bottleneck) our method provides the speedup with significant benefits: the model architecture doesn't change, retraining isn't required, and most importantly, \emph{the output distribution is guaranteed to stay the same}. Our method is easy to implement, and can be used to speedup inference using out-of-the-box models without developing and evaluating custom schemes.

There are several directions for follow up research, importantly, further investigating the compatibility of speculative decoding with beam search (see \cref{sec:beam}).
Also, while our method yields substantial speedups with existing off-the-shelf approximation models, greater improvements might be obtained via custom approximation models (\cref{sec:approximation_models}), such as those with custom architectures (e.g. custom sizes, non-autoregressive models, or various heuristics) or with custom training procedures (e.g. standard distillation with soft targets from $M_p$, or optimizing $M_q$ for $\alpha$ directly).
It could also be interesting to explore a hierarchical version of the algorithm, where the approximation model is itself accelerated by an even faster model, which could allow for more capable approximation models.
In this work we fixed the approximation model and the number of guesses $\gamma$ throughout inference, but varying them during inference could yield additional improvements (\cref{sec:optimal_gamma}).
In our experiments we always performed the same standardization on the distributions generated by the approximation model as the desired one for the target model (\cref{sec:standardized_sampling}), but further improvements might be obtained by applying different transformations.
We tested speculative decoding only in the text modality, but it might work well in other domains (e.g. images) which would be interesting to experiment with.

Finally, we note that \emph{stochastic speculative execution} and \emph{speculative sampling} can be helpful outside the scope of \emph{speculative decoding} from autoregressive models. For example, given two slow functions, $f(x)$ and $g(y)$ such that $f(x)$ generates a distribution from which $g$'s input is sampled, we could use our method to run $f$ and $g$ in parallel. This setup might arise e.g. in physics simulations, or in reinforcement learning where $f$ is a large model that produces a distribution on actions, and $g$ is the world simulation, which would be interesting to explore.

\section*{Acknowledgments}
We would like to extend a special thank you to YaGuang Li for help with everything LaMDA related and for calculating the LaMDA figures in the paper, and to Blake Hechtman for great insights and help with XLA.
We would also like to thank the reviewers for insightful comments, as well as Asaf Aharoni, Reiner Pope, Sasha Goldshtein, Nadav Sherman, Eyal Segalis, Eyal Molad, Dani Valevski, Daniel Wasserman, Valerie Nygaard, Danny Vainstein, the LaMDA and Theta Labs teams at Google, and our families.


\bibliography{references}
\bibliographystyle{icml2023}

\newpage
\appendix
\onecolumn

\appendix

\section{Appendix}

\subsection{Correctness of Speculative Sampling}
\label{appendix:correctness}

We will now show that for any distributions $p(x)$ and $q(x)$, the tokens sampled via \emph{speculative sampling} from $p(x)$ and $q(x)$ are distributed identically to those sampled from $p(x)$ alone. Let $\beta$ be the acceptance probability (\cref{def:beta_prefix}).

Note that as $p'(x) = norm(max(0, p(x) - q(x))) = \frac{p(x) - min(q(x), p(x))}{\sum_{x'}(p(x') - min(q(x'), p(x')))} = \frac{p(x) - min(q(x), p(x))}{1 - \beta}$, the normalizing constant for the adjusted distribution $p'(x)$ is $1 - \beta$, where the last equation follows immediately from \cref{lemma:dlk} and \cref{thm:dlk_beta}.

Now:

$$
    P(x=x') = P(guess\ accepted, x=x') + P(guess\ rejected, x=x')
$$

Where:

$$
    P(guess\ accepted, x=x') = q(x')\min(1, \frac{p(x')}{q(x')}) = \min(q(x'), p(x'))
$$

And:

$$
    P(guess\ rejected, x=x') = (1 - \beta)p'(x')  = p(x') - \min(q(x'), p(x'))
$$

Overall:

$$
    P(x=x') = \min(p(x'), q(x')) + p(x') - \min(p(x'), q(x')) = p(x').
$$

As desired. $\square$

\subsection{Speculative Sampling vs. Rejection Sampling}
\label{sec:vs_rejection_sampling}
Rejection sampling is the following iterative sampling procedure that looks superficially similar to ours:

\begin{enumerate}
    \item Sample $x \sim q(x)$ and $r \sim U(0,1)$.
    \item If $r < \frac{p(x)}{Mq(x)}$ return $x$.
    \item Go to 1.
\end{enumerate}

Where $M = max_{x}\frac{p(x)}{q(x)}$. We could employ a non-iterative version of rejection sampling instead of speculative sampling - specifically go through steps 1 and 2 above, and otherwise sample from an \emph{unmodified} $p(x)$ directly. That would be much less efficient than our method though. Specifically, the expected accept probability here is $E_{x \sim q(x)}\frac{p(x)}{Mq(x)} = \sum_{x}p(x)\min_{x'}\frac{q(x')}{p(x')} \le \sum_{x}p(x)\min(1, \frac{q(x)}{p(x)}) = \sum_{x}\min(p(x), q(x)) = \alpha$ is (potentially much) lower than the expected accept probability in our method $\alpha$.


\subsection{Theoretical Predictions vs. Empirical Runtimes}
\label{sec:theory_vs_emp}

\cref{tbl:t5 reality vs theory} compares the expected runtime improvements based on \cref{thm:total_walltime} to the empirically measured runtimes from \cref{tbl:t5 model results}.
We estimated the values of $c$ for the various models based on profiler traces.
We can see that the theoretical predictions mostly match the measured runtimes.
The larger differences are due to: (1) optimization differences between our implementation and the baseline, and (2) the simplifying assumption that the $\beta$s are i.i.d. being only an approximation (see \cref{sec:num_generated_tokens}).

\begin{table}[ht]
  \caption{Expected improvement factor (\textsc{Exp}) vs. empirically measured improvement factor (\textsc{Emp}).}
  \label{tbl:t5 reality vs theory}
\vskip 0.15in
\begin{center}
\begin{small}
\begin{sc}
\begin{tabular}{llccccrrr}
\toprule
Task & $M_q$ & Temp & $\gamma$ & $\alpha$ & $c$ & Exp & Emp \\
\midrule
EnDe & T5-small & 0 & 7 &  0.75  & 0.02 & 3.2 & 3.4 \\
EnDe & T5-base & 0 & 7 &  0.8  & 0.04 & 3.3 & 2.8 \\
EnDe & T5-large & 0 & 7 &  0.82  & 0.11 & 2.5 & 1.7 \\
EnDe & T5-small & 1 & 7 &  0.62  & 0.02 & 2.3 & 2.6 \\
EnDe & T5-base & 1 & 5 &  0.68  & 0.04 & 2.4 & 2.4 \\
EnDe & T5-large & 1 & 3 &  0.71  & 0.11 & 2.0 & 1.4 \\
\midrule
CNNDM & T5-small & 0 & 5 &  0.65  & 0.02 & 2.4 & 3.1 \\
CNNDM & T5-base & 0 & 5 &  0.73  & 0.04 & 2.6 & 3.0 \\
CNNDM & T5-large & 0 & 3 &  0.74  & 0.11 & 2.0 & 2.2 \\
CNNDM & T5-small & 1 & 5 &  0.53  & 0.02 & 1.9 & 2.3 \\
CNNDM & T5-base & 1 & 3 &  0.55  & 0.04 & 1.8 & 2.2 \\
CNNDM & T5-large & 1 & 3 &  0.56  & 0.11 & 1.6 & 1.7 \\
\bottomrule
\end{tabular}
\end{sc}
\end{small}
\end{center}
\vskip -0.1in
\end{table}


\subsection{Application to Beam Search}
\label{sec:beam}
Our method can be applied, with some performance penalty, to beam search sampling. Given the original beam width $w$, we can perform beam search with the approximation model $M_q$ and beam width $u \geq w$ for $\gamma$ steps. Then, we can use $M_p$ to check all of the candidates in parallel (costing a compute budget of $(w + u\gamma)$ runs of $M_p$). Finally, for each step, we can accept the guesses of $M_q$ as long as $top_w(M_p) \subseteq top_u(M_q)$ to get identical results to regular beam search with $M_p$ alone (with a more elaborate procedure we could also accept cases where the candidates we got happen to have higher probabilities than those of $M_p$ alone). The analysis of our method in this setting is more involved and we leave it for future work.

\subsection{Lenience}
A strong property of \cref{alg:sample_with_approx} is that the output distribution is guaranteed to remain unchanged. That said, if we're willing to allow some changes, with nice guarantees, we can get further inference speed improvements. To further motivate this, note that when we train two models with identical architectures and sizes on the same dataset, the generated probability distributions will not be identical, so some lenience might make sense. Note that the results in this paper except for this section use the strictest version of \cref{alg:sample_with_approx} and don't allow lenience of any kind.

We could include a lenience parameter $l \in [0, 1]$ and multiply $q(x)$ by $l$ before comparing with $p(x)$ in \cref{alg:sample_with_approx}. This still maintains the nice guarantee that no token can be sampled with probability greater than $\frac{p(x)}{l}$. This means for example, that with $l = \frac{1}{10}$ no token can be sampled with more than $10$X its ground truth probability, so we can guarantee that extremely rare tokens will remain extremely rare (there is no guarantee on the minimum probability, so lenience could hurt the diversity of the samples).

Specifically, with a lenience factor $l$ we have $\alpha = E_{x \sim q(x)}\left\lbrace
\begin{tabular}{@{}p{0.5cm}p{2cm}} 
$1$ & $lq(x) \le p(x)$ \\
$\frac{p(x)}{lq(x)}$ & $lq(x) > p(x)$ \\
\end{tabular}\right. = E_{x \sim q(x)}\frac{p(x)}{max(p(x), lq(x))} = \sum_x\frac{p(x)q(x)}{max(p(x), lq(x))} = \frac{1}{l}\sum_x\min(p(x), lq(x)) = \sum_x\min(\frac{p(x)}{l}, q(x))$.

\cref{tbl:lenience} shows $\alpha$ values for different values of $l$ when $M_p$ is T5-XXL (11B) and $M_q$ is T5-small (77M). With $c = 0.015$, using lenience values of 1, 0.5, 0.3, and 0.1 (meaning that no token can be sampled with probability greater than 1X, 2X, 3X and 10X of the ground truth) we get improvement factors of 2.5X, 3.1X, 3.6X, and 5X respectively.

\begin{table}[ht]
  \caption{$\alpha$ values for various values of $l$ with standard sampling where $M_p$ is T5-XXL (11B) on the EnDe translation task.}
  \label{tbl:lenience}
\vskip 0.15in
\begin{center}
\begin{small}
\begin{sc}
\begin{tabular}{lcccr}
\toprule
$M_q$ & $l=1$ & $l=0.5$ & $l=0.3$ & $l=0.1$ \\
\midrule
    Unigram         & 0.07 & 0.1    & 0.11  & 0.16 \\
    Bigram          & 0.19 & 0.23   & 0.25  & 0.32\\
    T5-small (77M)  & 0.62 & 0.71   & 0.76  & 0.84\\
    T5-base (250M)  & 0.68 & 0.8    & 0.83  & 0.90\\
\bottomrule
\end{tabular}
\end{sc}
\end{small}
\end{center}
\vskip -0.1in
\end{table}

Note that when using temperature = 0 (i.e. argmax sampling), we can no longer use lenience as above. Instead, we could allow some lenience before standardizing the distributions. For example, we could accept the token $x$ sampled from $M_q$ in case $p(x) \leq l \cdot max(p)$. In this case, we measure similar empirical increases in $\alpha$ values to those with temperature = 1. For example, when using lenience values of 1, 0.5, 0.3, and 0.1 for $M_p$ T5-XXL $M_q$ T5-small for English-German translation, we get $\alpha$ values of 0.75, 0.75, 0.8, 0.87. Taking for example $c=0.015$ and $\gamma=8$ we get speed improvement factors of 3.3X, 3.3X, 3.9X, and 4.9X respectively\footnote{In this case, unlike in the standard sampling case shown in \cref{tbl:lenience}, a lenience factor of 0.5 doesn't improve the speed-up.}.


\end{document}